\documentclass{article}
\usepackage[titletoc]{appendix}
\usepackage{times, natbib}
\usepackage{graphicx} 
\usepackage{subfigure} 
\usepackage{multirow}
\usepackage{delarray,mathtools}
\usepackage{natbib}

\usepackage{algorithm, bm}

\usepackage{algorithmic}

\usepackage{hyperref}

\usepackage{amsmath, amssymb, amsthm, tikz}
\usepackage{blkarray}
\usepackage{xcolor}
\usepackage[all,cmtip]{xy}

\newcommand{\R}{\mathbb{R}}

\newcommand{\cone}{\mathrm{cone}}

\newcommand{\bh}{\mathbf{h}}

\newcommand{\N}{\mathcal{N}}

\newcommand{\Span}{\mathrm{span}}

\newcommand{\rank}{\mathrm{rank}}

\newcommand{\col}{\mathrm{col}}
\newcommand{\bv}{\mathbf{v}}
\newcommand{\zero}{\mathbf{0}}

\theoremstyle{theorem}
\newtheorem{thm}{Theorem}

\theoremstyle{theorem}
\newtheorem{prop}{Proposition}
\theoremstyle{theorem}
\newtheorem{cor}{Corollary}

\theoremstyle{lemma}

\theoremstyle{remark}
\newtheorem{remark}{Remark}

\theoremstyle{definition}
\newtheorem{defn}{Definition}

\theoremstyle{example}
\newtheorem{ex}{Example}

\theoremstyle{definition}

\newcommand\undermat[2]{%
  \makebox[0pt][l]{$\smash{\underbrace{\phantom{%
    \begin{matrix}#2\end{matrix}}}_{\text{$#1$}}}$}#2}

\begin{document} 

\title{A Characterization of the Non-Uniqueness of Nonnegative Matrix Factorizations}


\author{Finale Doshi-Velez, \hskip0.5cm Weiwei Pan, \\ Harvard University, \\Cambridge, MA 02138 USA \hskip0.5cm Harvard University, Cambridge, MA 02138 USA}

\author{
  Pan, Weiwei\\
  Harvard University, Cambridge, MA 02138 USA
  \and
  Doshi-Velez, Finale\\
  Harvard University, Cambridge, MA 02138 USA
}

\maketitle

\begin{abstract} 
Nonnegative matrix factorization (NMF) is a popular dimension reduction technique that produces interpretable decomposition of the data into parts.  However, this decompostion is not generally identifiable (even up to permutation and scaling).  While other studies have provide criteria under which NMF is identifiable, we present the first (to our knowledge) characterization of the \emph{non-identifiability} of NMF.  We describe exactly when and how non-uniqueness can occur, which has important implications for algorithms to efficiently discover alternate solutions, if they exist.  
\end{abstract}

\section{Introduction}
Nonnegative matrix factorization (NMF) is a technique that is widely applied in the analysis of high dimensional data due to its ability to automatically extract localized features from the data that is interpretable in context of the application \citep{Lee}. Typically, nonnegative matrix factorization is presented as an optimization problem. That is, given some set of non-negative data realized as column vectors in a matrix $S\in\R^{M\times N}$, we seek a pair of nonnegative matrices $H \in \R^{M\times R}$, $W \in \R^{R\times N}$ such that the product $HW$ best approximates $S$. The solution $(H, W)$ then gives a decomposition of the data into ``parts" in the sense that the data in $S$ is represented as sums of the column vectors of $H$, weighted by the column vectors in $W$. 

The nonnegativity constraints on the factors $W$ and $H$ allows both the column vectors in $W$ and the reconstruction of $S$ as linear combinations $H$ to be meaningfully interpreted when the data corresponds to real, physical quantities. For example, in image analysis, where $W$ yields a set of basis image features and $H$ a set of weights, decompositions by NMF might be preferred over tools like PCA due to the difficulties of interpreting negative basis entries and negative weights. One of the fist examples of image feature extraction via NMF is provided by Lee and Seung \citep{Lee}, who decomposed facial images into features such as eyes, noses and lips.  

A notable property of NMF is that it is ill-posed.  As one example, given any pair of nonnegative factors $(W, H)$ of $S$, any invertible matrix $Q$ produces an equivalent pair of factors $(WQ, Q^{-1}H)$. When $Q$ is a monomial matrix---a permutation with positive entries, for example, $(WQ, Q^{-1}H)$ is another nonnegative factorization of $S$.  More significantly, there are cases where non-monomial choices for $Q$ results in alternate nonnegative factorizations, that is, the parameters of the underlying model for the data is not identifiable through NMF. The non-uniqueness of NMF is generally considered to be a weakness of the method and a number of studies exist in literature which give conditions on or preprocessing for the data matrix $S$ under which the resulting NMF is unique \citep{Donoho, Huang, Laurberg, Gillis}, or otherwise ensure uniqueness of factorization by the addition priors on the factor and regularization terms in the objective function \citep{Hoyer, Kim}. 

However, it is generally hard to ascertain which, if any, of the assumptions above are relevant to a particular data set.  Gillis et. al. points out that that different factorizations of data in the context of topic modeling yield different topics \citep{Gillis}.  Practitioners often resort to heuristic approaches, such as running an optimization procedure random restarts, to discover variation in NMF solutions. For example, for an application of NMF clustering to protein functional groupings, an ensemble of diverse solutions of NMF solutions provide useful features for various downstream tasks \citep{Greene}.  Aside from these heuristic explorations, there is a lack of systematic examination of the non-uniqueness of NMF.  In particular, (1) the size of the set of exact solutions of an NMF problem, and (2) the underlying structure of the solution set of an NMF problem are both currently poorly understood.  

In this paper, we describe a framework for characterizing the set of \emph{all} exact solutions to an NMF problems.  Specifically, we show that non-uniqueness can occur exactly in one of three ways, each of which have non-trivial examples.  We prove properties about these types and discuss how these insights may be used to characterize and find diverse solutions in approximate NMF settings.

\section{Brief Review of Nonnegative Matrix Factorization}\label{sec:nmf}
Let $S$ be a nonnegative matrix in $\R_+^{M\times N}$.  We will denote
the $(m, n)$-th entry by $S_{m, n}$, the $m$-th row of a matrix $S$ by
$S_{m,\cdot}$, and the $n$-th column of $S$ by $S_{\cdot,n}$.  Let
$R\in \mathbb{N}_+$ with $R\leq \min\{ N, M\}$ be the desired rank of
the factorization. A \emph{rank-$R$ nonnegative matrix factorization}
of the matrix $S$ is a pair of factors $W\in \R_+^{M\times R}$, $H \in
\R_+^{R\times N}$ such that
\begin{align}\label{eq:obj}
(W, H) = \underset{W, H \geq 0}{\mathrm{argmin}} L( S - WH )  
\end{align}
where $L(\cdot)$ represents a loss function, such as the matrix
Frobenius norm $\|S - WH\|_F$.  If $S = WH$, we call the pair $(W, H)$
an \emph{exact NMF}; the minimum rank $R$ such that $S$ admits an
exact NMF is called the \emph{nonnegative rank of $S$} and is denoted
$\rank_+(S)$.

Each column of $S$ is thus a nonnegative linear combination (weighted
by some column in $H$) of the columns of $W$.  
This insight results in the geometric interpretation of NMF: there
exists a bijection between nonnegative factorization of $S$ and
simplicial cone 
\begin{align}
\cone(W) = \left\{W\bm{\alpha}: \bm{\alpha}\in \R_+^R \right\}
\end{align}
in the positive orthant $\R_+^{m}$ containing the column vectors of
$S$.  The columns of $W$ are called the \emph{generators} of the cone
$\cone(W)$.  

Given a solution $(W,H)$, we can trivially find other solutions
$(WM,M^{-1}H)$ that result in the same loss.  An NMF is \emph{unique}
if all solutions can be represented by a monomial matrix $M$, otherwise, we call the NMF \emph{non-identifiable}.  
Below we review two sufficient conditions for uniqueness
(Theorems~\ref{thm:sep} and~\ref{thm:spread}) and two necessary
conditions for uniqueness (Theorems~\ref{thm:bclose}
and~\ref{thm:huang}).  

\begin{thm}\label{thm:sep}\citep{Donoho} Fix $P, A\in \N$ such that $A >2$. The NMF $S = WH$ is unique in $\col(W)$ provided that the following conditions are satisfied.
\begin{enumerate}
\item[{[R1]}] \emph{Generative Model.} We have $R = A\cdot P$. For each pair $1\leq a \leq A$ and $1\leq p\leq P$ there exists a unique $1\leq r(p, a) \leq R, $ such that each column of $S$ can be represented as
\begin{align}
S^{n} = \sum_{p=1}^P\sum_{a=1}^A H_{r(p, a), n}W^{r(p, a)},
\end{align}
where $H_{r(p, a), n}$ denotes the entry in $H$ at the $r(p, a), n$-position and $W^{r(p, a)}$ denotes the $r(p, a)$-th column of $W$. We say that $W_{r(p, a)}$ is the $p$-th part in the $a$-th articulation.

\item[{[R2]}] \emph{Complete Factorial Sampling.} For any set of indices $I = \{r(1, a_1), \ldots, r(P, a_P)\}$, where each $1\leq a_p\leq A$, we have some $1\leq n\leq N$ where
\begin{align}
H_{r, n} &\neq 0, \;\;\;\; r \in I\\
H_{r, n} &= 0, \;\;\;\; r \notin I
\end{align}

\item[{[R3]}] \emph{Separability.} For each pair $1\leq a \leq A$ and $1\leq p\leq P$, there exist $1\leq m(p, a)\leq M$ such that 
\begin{align}
W_{m(p, a), r} &\neq 0,\;\;\;\; r = r(p, a)\\
W_{m(p, a), r} &= 0,\;\;\;\; r \neq r(p, a)
\end{align}
\end{enumerate}
\end{thm}

A matrix $S \in\R_+^{M\times N}$ satisfying the conditions [R1] through [R3] of Theorem \ref{thm:sep} is called a \emph{separable factorial articulation family}. Intuitively, a separable factorial articulation family is a collection of nonnegative combinations of a set of $P$ parts, each in $A$ number of possible articulations (realized as $R = P\cdot A$ number vectors in $\R^M$). The presence of any articulation of a particular part in a given combination can be detected by checking the value of a uniquely associated position in the resulting vector. The complete factorial sampling condition ensures that all possible combinations of parts and articulations appear in $S$ (see Example (\ref{ex:type2})). 

We note that the condition $A>2$ is not included in the statement of Theorem \ref{thm:sep} as it appears in \cite{Donoho}. However, we find that when $A = 2$, one obtains examples of separable factorial articulation families with $S = WH = W'H'$, where $\col(W') \subset \col(W)$. 

\begin{ex}
Let $S = 
\left( \begin{matrix}
1 &0 & 1 & 0\\
0 &1 & 0 & 1\\
1 &1 & 0 & 0\\
0 &0 & 1 & 1\\
0 &0 & 0 & 0\\
\end{matrix}\right)$. Note that $S$ admits the following factorization
\begin{align}
S =  \left( \begin{matrix}
1 &0 & 0 & 0\\
0 &1 & 0 & 0\\
0 &0 & 1 & 0\\
0 &0 & 0 & 1\\
\undermat{W}{0 &0 & 0 & 0}\\
\end{matrix}\right) \left( \begin{matrix}
1 &0 & 1 & 0\\
0 &1 & 0 & 1\\
1 &1 & 0 & 0\\
\undermat{H}{0 &0 & 1 & 1}\\
\end{matrix}\right)
\end{align}
\vskip0.2cm
Thus, $S = WH$ is a complete factorial family: $W$ is separable with two parts and two articulations each (without loss of generality, we may interpret the first two columns of $W$ to be part 1 and the last two to be part 2), and $H$ is a complete factorial sampling. On the other hand, we also have the factorization $S = S I_{4\times 4}$. Thus, $S$ has two distinct factorizations in $\col(W)$.
\vskip0.01cm
${}$ \hfill $\square$
\end{ex}

\begin{thm}\citep{Laurberg}\label{thm:spread}
Let $R = \rank(W)$. The NMF $S = WH$ is unique in $\R^M$ provided that the following conditions are satisfied.
\begin{enumerate}
\item[{[R1]}] \emph{Sufficiently Spread.} For each $1 \leq r \leq R$, there is a corresponding $1 \leq n\leq N$ such that 
\begin{align}
H_{r, j} &\neq 0, \;\;\;\; j = n\\
H_{r,j} &= 0, \;\;\;\; j \neq n
\end{align}
\item[{[R2]}] \emph{Strongly Boundary Close.} The matrix $W$ satisfies 
\begin{enumerate}
\item[{[A]}] \emph{Boundary Close.} For each $1 \leq r \leq R$, there is a corresponding $1\leq m\leq M$ such that 
\begin{align}
W_{i, r} &= 0, \;\;\;\; i = m\\
W_{i,r} &\neq 0, \;\;\;\; i \neq m
\end{align}
\item[{[B]}] There exist a permutation matrix $P$ such that for every $1 \leq m\leq R$, there is a set $\{r_{1}, \ldots, r_{R-m}\}$ fulfilling:
\begin{align}
(WP)_{m, r_j} = 0, \;\;\;\;  j\leq K-m;
\end{align}
and the matrix
\begin{align}
\left[ 
\begin{matrix}
(WP)_{m+1, r_1} &\ldots &(WP)_{m+1, r_{R-m}}\\
\vdots &\ddots &\vdots\\
(WP)_{R, r_1} &\ldots &(WP)_{R, r_{R-m}}\\
\end{matrix}
\right]
\end{align}
is invertible.
\end{enumerate}
\end{enumerate}
\end{thm}

Note that Theorems \ref{thm:sep} and \ref{thm:spread} are sufficient conditions. The following are two necessary conditions for uniqueness.

\begin{thm}\label{thm:bclose} \citep{Laurberg}
If the NMF $S=WH$ is unique, then $W$ is boundary close.
\end{thm}

\begin{thm}\label{thm:huang} \citep{Huang}
Let 
\begin{align}
\mathcal{M}_r &= \{1\leq m \leq M : W_{m, r} \neq 0 \}\\
\mathcal{N}_r &= \{ 1\leq n \leq N: H_{r, n} \neq 0\}
\end{align}
If the NMF $S=WH$ is unique, then there does not exist $1\leq r_1< r_2 \leq R$ such that $\mathcal{M}_{r_1}\subset\mathcal{M}_{r_2}$ or $\mathcal{N}_{r_1}\subset\mathcal{N}_{r_2}$.
\end{thm}


\section{Characterizing Non-identifiability in NMF}
At first glance, the results cited in Section \ref{sec:nmf} may appear to apply to classes of NMFs which satisfy ad-hoc and unrelated conditions. In the following, we describe a intuitive and general framework in which existing uniqueness results can be related.

\begin{thm}\label{thm:framework}
An exact NMF problem has non-unique solutions exactly when the column vectors of the matrix $S \in\R_+^{M\times N}$ lie in the intersection of two or more simplicial cones in $\R_+^M$, each with $R$ generators. This happens in just three ways:
\begin{enumerate}
\item[]\textbf{Type I} the column vectors of $S$ lies in the intersection of two (or more) simplicial cones whose generators spans the same $R$-dimensional subspace. That is,  $\rank_+(S) = \rank(S) = R$ and 
\begin{align}
S = WH = WQQ^{-1}H
\end{align}
where $Q\in \R^{R\times R}$ is a change of basis matrix.
\item[] \textbf{Type II} the column vectors of $S$ lies in the intersection of two (or more) simplicial cones who are each unique in the $R$-dimensional subspaces spanned by their generators. That is, $R = \rank_+(S) > \rank(S)$ and
\begin{align}
S = WH = QWH'
\end{align}
where $Q \in \R^{M\times M}$ can be completed into a change of basis matrix. Furthermore, if $S = WH = WPP^{-1}H$, then $P$ is a monomial matrix in $\R^{R\times R}$. 
\item[]\textbf{Type III} $S$ is both Type I and Type II. That is,
\begin{align}
S = WH = Q_1WQ_2Q_2^{-1}H'
\end{align}
where $Q_2\in \R^{R\times R}$, $Q_1 \in \R^{M\times M}$ are change of basis matrices.
\end{enumerate}
\end{thm}
It follows immediately from Theorem \ref{thm:framework}, that the uniqueness results given by Laurberg et. al. and Huang et. al. \citep{Laurberg, Huang} apply to non-identifiability of Type I. We will prove in Section \ref{sec:sep} that the separable factorial articulation families defined by Donoho and Stodden \citep{Donoho} may be used to understand a subclass of models which are non-identifiable of Type II.

The proof of Theorem \ref{thm:framework} is straightforward. It is clear that Types I to III partitions all cases of non-identifiability in NMF problems. The harder task, to which we devote the remainder of this section, is to demonstrate that each class contains non-trivial examples. 

The following is an example of Type I non-identifiability from Laurberg et. al. \citep{Laurberg}.

\begin{ex}\label{ex:type1}
\emph{Non-identifiability of Type I.}  The column vectors of $W$, defined below, spans a 3-dimensional subspace in $\R^6$.
\begin{align}
H = \left(\begin{matrix}
0.5 & 1 & 1& 0.5 & 0 & 0\\
1 & 0.5 & 0 & 0 & 0.5 & 1\\
0 & 0 & 0.5 & 1 & 1 & 0.5
\end{matrix}\right), \;\;
W = H^T
\end{align}
Choose a change of basis matrix $Q$ defined by
\begin{align}
Q = Q^{-1} =\frac{1}{3} \left(\begin{matrix}
-1 & 2& 2\\
2 & -1 & 2\\
2 & 2 & -1
\end{matrix}\right).
\end{align}
Let $S = WH = (WQ)(Q^{-1}H)$. Then, the set of column vectors in $S$ lies in the intersection $\cone(W) \cap \cone(WQ)$. 
\vskip0.01cm
${}$ \hfill $\square$
\end{ex}

For Type II, generators of the intersecting cones containing $S$ correspond to bases for different $R$-dimensional subspaces in $\R^M$. Thus, two sets of such generators can be non-uniquely extended to bases of $\R^M$ and thus differ by a matrix, $Q \in \R^{M\times M}$.

\begin{ex} \label{ex:type2} \emph{Non-identifiability of Type II.}
\noindent Let
\begin{align}\scriptsize
S = \left( 
\begin{array}{rrrrrrrrr}
1 &0 &0 & 1 &0 &0 & 1 &0 &0\\
0 &1 & 0 & 0 &1 & 0 & 0 &1 & 0\\
0 &0 & 1 & 0 &0 & 1 & 0 &0 & 1\\
\hline
1 &1 & 1 & 0 &0 & 0 & 0 &0 & 0\\
0 &0 & 0 & 1 &1 & 1 & 0 &0 & 0\\
0 &0 & 0 & 0 &0 & 0 & 1 &1 & 1\\
\hline
\undermat{\text{Invariant row}}{1 & 1 &1 & 1& 1&1 & 1& 1&1}\\
\end{array}\right)
\end{align}
\vskip0.1cm
We can easily verify that $\rank(S) = 8$ and that $S$ lies in the span of its first eight column vectors. Let $\tilde{S}$ be the following matrix obtained from $S$ by removing the invariant row.
\begin{align}\scriptsize
\tilde{S} = \left( \begin{array}{rrrrrrrrr}
1 &0 &0 & 1 &0 &0 & 1 &0 &0\\
0 &1 & 0 & 0 &1 & 0 & 0 &1 & 0\\
0 &0 & 1 & 0 &0 & 1 & 0 &0 & 1\\
\hline
1 &1 & 1 & 0 &0 & 0 & 0 &0 & 0\\
0 &0 & 0 & 1 &1 & 1 & 0 &0 & 0\\
0 &0 & 0 & 0 &0 & 0 & 1 &1 & 1
\end{array}\right).
\end{align}
We show that $\tilde{S}$ is a separable factorial articulation family.
\begin{enumerate}
\item $\tilde{S}$ is non-negatively generated by a basis set, $\tilde{W}$, in two parts $p_1, p_2$ and three articulations $a_1, a_2, a_3$.
\begin{align}\tiny
\left\{
\underbrace{\left(
\begin{array}{c}
1 \\
0\\
0\\
\hline
0\\
0\\
0\\
\end{array}
\right)}_{p_1, a_1},
\underbrace{\left(
\begin{array}{c}
0 \\
1\\
0\\
\hline
0\\
0\\
0\\
\end{array}
\right)}_{p_1, a_2},
\underbrace{\left(
\begin{array}{c}
0 \\
0\\
1\\
\hline
0\\
0\\
0\\
\end{array}
\right)}_{p_1, a_3},
\underbrace{\left(
\begin{array}{c}
0 \\
0\\
0\\
\hline
1\\
0\\
0\\
\end{array}
\right)}_{p_2, a_1},
\underbrace{\left(
\begin{array}{c}
0 \\
0\\
0\\
\hline
0\\
1\\
0\\
\end{array}
\right)}_{p_2, a_2},
\underbrace{\left(
\begin{array}{c}
0 \\
0\\
0\\
\hline
0\\
0\\
1\\
\end{array}
\right)}_{p_2, a_3}
\right\}
\end{align}
We have that $\tilde{S} = \tilde{W}\tilde{H}$, where $\tilde{H} = \tilde{S}$.
\item For each column of $\tilde{S}$, the nonzero entry in each row indicates the presence of the corresponding generator. That is, $\tilde{W}$ is separable.
\item $\tilde{S}$ contains all possible combinations of articulations and parts. That is, $\tilde{H}$ is a complete factorial sampling.
\end{enumerate}
Thus, by Theorem \ref{thm:sep}, we have that $\tilde{S}$ is contained in an unique simplicial cone, $\cone(\tilde{W})$. 
\vskip0.1cm
\noindent Now, various ways of distributing the invariant row in $S$ to the two parts of $\tilde{W}$ will result in different factorizations of $S$:
\begin{align}
S &= \tiny\underbrace{\left( \begin{matrix}
1 &0 & 0 & 0 & 0 & 0\\
0 &1 & 0 & 0 & 0 & 0\\
0 &0 & 1 & 0 & 0 & 0\\
\hline
0 &0 & 0 & 1 & 0 & 0\\
0 &0 & 0 & 0 & 1 & 0\\
0 &0 & 0 & 0 & 0 & 1\\
\hline
1 & 1 & 1 & 0 & 0 & 0\\
\end{matrix}\right)}_{W_1}
H = \tiny\underbrace{\left( \begin{matrix}
1 &0 & 0 & 0 & 0 & 0\\
0 &1 & 0 & 0 & 0 & 0\\
0 &0 & 1 & 0 & 0 & 0\\
\hline
0 &0 & 0 & 1 & 0 & 0\\
0 &0 & 0 & 0 & 1 & 0\\
0 &0 & 0 & 0 & 0 & 1\\
\hline
0 & 0 & 0 & 1 & 1 & 1\\
\end{matrix}\right)}_{W_2}
H
\end{align}
where $H = \tilde{S}$. 
Now note that the generators in the above factorizations of $S$ are related by a matrix $Q$ in $\mathbb{R}^{7\times 7}$ which can be extended to a change of basis matrix on $\R^7$:
\begin{align} W_1 = &
\underbrace{\scriptsize\left( \begin{matrix}
1 &0 & 0 & 0 & 0 & 0 & 0\\
0 &1 & 0 & 0 &0 & 0 & 0\\
0 &0 & 1 & 0 &0 & 0 & 0\\
0 &0 & 0 & 0 &-1 & -1 & 1\\
0 &0 & 0 & 0 &1 & 0 & 0\\
0 &0 & 0 & 0 &0 & 1 & 0\\
1 & 1 & 1 &0 & 0 & 0 & 0\\
\end{matrix}\right)}_{Q} W_2
\end{align}
It this case, it is clear that there does not exist an invertible $Q\in \R^{6\times 6}$ such that $W_1Q = W_2$. Since the column-echelon forms of $W_1$ and $W_2$ are distinct, they span different subspace of $\R^{M}$. 

Also of interest is the fact that we can completely describe the class of solutions to this NMF problem. Namely, every basis factor matrix $W$ for $S$ can be expressed as:
\begin{align}\scriptsize
\left( \begin{matrix}
1 &0 & 0 & 0 & 0 & 0\\
0 &1 & 0 & 0 & 0 & 0\\
0 &0 & 1 & 0 & 0 & 0\\
0 &0 & 0 & 1 & 0 & 0\\
0 &0 & 0 & 0 & 1 & 0\\
0 &0 & 0 & 0 & 0 & 1\\
1-t & 1-t & 1-t & t & t & t\\
\end{matrix}\right)
\end{align}
for $t\in \R_+$. In Section \ref{sec:sep}, we give a characterization of when such ``closed-form" solutions exist for a general class of NMFs.
${}$ \hfill $\square$
\end{ex}

\begin{remark}
In Example \ref{ex:type2}, we chose to append an invariant row to a separable factorial articulation family whose unique factorization is known. In practice, such a row would have been removed from the data set in preprocessing. In Section \ref{sec:sep}, we describe how 
\end{remark}

In Example \ref{ex:type2}, we notice that $S$ is contained in a proper subspace of $\col(W_1)$, as well as $\col(W_2)$. Thus, we must have that $\rank(S)$ is less than $R = \dim(\col(W_i))$. However, note that the nonnegative rank of $S$ may still be $R$. In fact, this is precisely what we will prove in Section \ref{sec:sep}, where we will make precise and generalize the intuition behind this example. In literature, it is typical to make the assumption that $\rank(M) = \rank_+(M) = R$ in order to eliminate non-identifiability of Type II.

For Type III, we take a combination of approaches in Examples (\ref{ex:type1}) and (\ref{ex:type2}) to engineer matrices with non-unique factorizations.

\begin{ex}\label{ex:type3} Non-identifiability of Type III.
Let
\begin{align}
M =\small \left( 
\begin{array}{rrrr}
2 &1&2 & 1\\
2 &3 & 2 & 3\\
1 &1 & 0 & 0\\
0 &0 & 1 & 1\\
\hline
\undermat{\text{Invariant row}}{1 & 1 &1 & 1}\\
\end{array}\right).
\end{align}
\vskip0.15cm
As expected, we obtain different factorizations by first factoring the matrix after removing the invariant row, and then distributing the invariant row amongst column vectors in the basis matrix
\begin{align}
M &= \small\underbrace{\left( \begin{matrix}
2 &1 & 0 & 0\\
2 &3 & 0 & 0\\
0 &0 & 1 & 0\\
0 &0 & 0 & 1\\
\hline
1 & 1 & 0 & 0\\
\end{matrix}\right)}_{W_1}
\underbrace{\left( \begin{matrix}
1 &0 & 1 & 0\\
0 &1 & 0 & 1\\
1 &1 & 0 & 0\\
0 &0 & 1 & 1\\
\end{matrix}\right)}_{H_1} \\
&=\small \underbrace{\left( \begin{matrix}
2 &1 & 0 & 0\\
2 &3 & 0 & 0\\
0 &0 & 1 & 0\\
0 &0 & 0 & 1\\
\hline
0 & 0 & 1 & 1\\
\end{matrix}\right)}_{W_2}
\underbrace{\left( \begin{matrix}
1 &0 & 1 & 0\\
0 &1 & 0 & 1\\
1 &1 & 0 & 0\\
0 &0 & 1 & 1\\
\end{matrix}\right)}_{H_2}
\end{align}
Here again, we have $W_1 = QW_2$, where 
\begin{align}
Q =\small \left( \begin{matrix}
1 &0 & 0 & 0\\
0 &1 & 0 & 0\\
0 &0 & -1 & 1\\
0 &0 & 1 & 0\\
0.25 &0.25 & 0 & 0\\
\end{matrix}\right).
\end{align}
Column reducing $W_1$ and $W_2$ into echelon form will show that $\col(W_1)\neq \col(W_2)$. Furthermore, once we fix the distribution of the invariant row represented by $W_2$, we obtain non-unique factorizations in the column space of $W_2$:
\begin{align}
M &= \small\underbrace{\left( \begin{matrix}
2 &1 & 0 & 0\\
2 &3 & 0 & 0\\
0 &0 & 1 & 0\\
0 &0 & 0 & 1\\
\hline
0 & 0 & 1 & 1\\
\end{matrix}\right)}_{W_2} 
H_2= \underbrace{\left( \begin{matrix}
1 &0 & 1 & 1\\
0 &1 & 2 & 2\\
0 &0 & 1 & 0\\
0 &0 & 0 & 1\\
\hline
0 & 0 & 1 & 1\\
\end{matrix}\right)}_{W'_2}
H_2
\end{align}
We have that $W'_2 = W_2Q$ and $H_2 = Q^{-1}H_2$, where
\begin{align}
Q = \small\left( \begin{matrix}
0.75 &-0.25 & 0.25 & 0.25\\
-0.5 &0.5 & 0.5 & 0.5\\
0 &0 & 1 & 0\\
0 &0 & 0 & 1\\
\end{matrix}\right)
\end{align}

${}$ \hfill $\square$
\end{ex}


\section{Characterizing Type I Non-identifiability}\label{sec:decomp}
In this section, we describe a specialized class of non-identifiable NMF models, $S = WH$, whose multiple factorizations lie in the same subspace of $\R^M$. For this, we will assume that $\rank(S) = \rank_+(S) = R$. Since necessary and sufficient conditions for the uniqueness of full-rank NMFs appear already in literature \citep{Laurberg, Huang}, we will instead focus on describing cases where the non-identifiability of an NMF model can be reduced to the non-identifiability of a sub-model. The goal of this study is to provide the theoretical foundation for characterizing the non-identifiability of NMF models by characterizing their sub-models.

Throughout this section, we will assume that $\col(W) \subset \Span\{g_i\}_{i=1}^T$, for some $T\geq R$, where $G = \{g_i\}_{i=1}^T$ is a linearly independent set of non-negative vectors in $\R^M$. That is, we assume that $S$ lies in the simplicial cone generated by the column vectors in $W$, with the latter a sub-cone of the simplicial cone generated by $G$. If $S$ is non-identifiable of Type I, then there exist $W' \in \R^{m\times r}$ whose column vectors generated a sub-cone of $\cone(G)$ that contains $S$. In short, we restrict the study of the Type I non-identifiability of $S = WH$ to the case where $S$ is contained in the intersection of multiple rank-$R$ sub-cones of a rank-$T$ nonnegative simplicial cone. 

The focus of factorizations of $S$ within a fixed super-cone is restrictive in appearance only; note that when $G = I_{M\times M}$, the set of sub-cones of $\cone(G)$ is indeed the set of rank-$R$ NMFs of $S$. In practice, working with $G\neq I_{M\times M}$ may be a reasonable approach to characterizing non-identifiability. We will see in later examples that, given a rank-$R$ matrix $S$, first computing a rank-$T$ factorization $S = GW_G$ with desirable qualities can simplify the task of characterizing the set of rank-$R$ factorizations contained in $\cone(G)$. 

For $B\in \R^{M\times R}$, it is clear that $\cone(B)\subset \cone(G)$ if and only if $B = GW$, where $W \in \R^{T\times R}$ is nonnegative. Thus, if $S$ is contained in two sub-cones of $\cone(G)$, then $S = GW_1H_ 1 = GW_2H_1$, where $W_i \in  \R^{T\times R}$ and $H_i \in  \R^{R\times N}$ are nonnegative. 

\begin{defn}\label{def:block}
We call $S = GWH$ a \emph{block-diagonal model} if there exist a permutation matrix $P\in \R^{R\times R}$ such that $WPPH$ is block diagonal, where each block cannot be further block-diagonalized by permutation matrices. 
\end{defn}

Note that for any permutation $P\in \R^{R\times R}$, the product $GWPPH$ is a 
a permuted version of $S$, that is, $ GWPPH = SP$, which we shall not distinguish from $S$. 

Note also that every NMF model can be realized as a block-diagonal model; in the case that $WH$ can only be block-diagonalized with one block, we call the model \emph{indecomposable}. A block-diagonal model can be equivalently defined by the individual block-diagonal structures of $W$ and $H$. 

\begin{prop}\label{prop:block}
$S = GWH$ is a block-diagonal model if and only if there is a permutation matrix $P$ such that 
\begin{enumerate}
\item $WP$ and $PH$ are block diagonal. We will denote the $k$-th block of $WP$ and of $PH$ by $\square^{W}_k$, $\square^{H}_k$, respectively. 
\item The blocks of $WP$ and $PH$ are not further block-diagonalizable by permutation matrices.
\item For each block $\square^{H}_k$, spanning rows $J = \{j_1,\ldots, j_{\alpha_k}\}$, one of the following attains:
\begin{enumerate}
\item for each block $\square^{W}_{k'}$, the columns spanned by $\square^{W}_{k'}$ is either contained in $J$ or is disjoint from it;
\item there is some block $\square^{W}_{k'}$, spanning columns $J' = \{j'_1, \ldots, j'_{\alpha'_{k'}}\}$, such that $J \subset J'$ and that, for each block $\square^{H}$, the rows spanned by $\square^{H}$ is either contained in $J'$ or is disjoint from it. 
\end{enumerate}
\end{enumerate}
\end{prop}

\begin{proof}
The proof follows straightforwardly from the definition of block-diagonal models.
\end{proof}

Block-diagonal models have the particularly important property that each such model admits a direct sum decomposition. 

\begin{cor}
Let $S = GWH$ be a block-diagonal model where $K$ is the number of blocks in $WH$. Then $S = \bigoplus_{k = 1}^K {}_kS$, where ${}_kS\in \R^{\alpha_k\times \delta_k}$ such that $\sum_{k=1}^K \alpha_k = M$ and $\sum_{k=1}^K \delta_k = N$.
\end{cor}

\begin{proof}
From Proposition \ref{prop:block}, we have that each block $\square^{WH}_k$ in $WH$ is the product of a pair of sub-matrices $W^{[j^k_1: j^k_{\beta_k}]}_{[i^k_1: i^k_{\alpha_k}]}$, $H_{[j^k_1: j^k_{\beta_k}]}^{[l^k_1: l^k_{\delta_k}]}$ of $W$, $H$, respectively.  
Thus, we obtain a direct sum decomposition of $S = GWH$ along blocks of $WH$, i.e. $\bigoplus_{k = 1}^K {}_kS = \bigoplus_{k = 1}^K G^{[i^k_1: i^k_{\alpha_k}]}W^{[j^k_1: j^k_{\beta_k}]}_{[i^k_1: i^k_{\alpha_k}]}H_{[j^k_1: j^k_{\beta_k}]}^{[l^k_1: l^k_{\delta_k}]}$.
\end{proof}

Intuitively, if $S$ is a block-diagonal model then it is the sum of independent NMF sub-models. Thus, the non-identifiability of $S$ can be traced to the non-identifiability of a sub-model of $S$. 

\begin{thm}
Let $S = GWH$ be a block-diagonal model in block diagonal form, where $HW$ has $K$ number of blocks. Let $S = \bigoplus_{k = 1}^K {}_kS$ be the corresponding decomposition of $S$ into sub-models. Then $S = GWH$ is non-identifiable if and only if some sub-model ${}_{k}S$ is non-identifiable. 
\end{thm}

\begin{proof}
Clearly, if $S$ has an unidentifiable submodel, say ${}_kS = G^{[i^k_1: i^k_{\alpha_k}]}W^{[j^k_1: j^k_{\beta_k}]}_{[i^k_1: i^k_{\alpha_k}]}H_{[j^k_1: j^k_{\beta_k}]}^{[l^k_1: l^k_{\delta_k}]}= G^{[i^k_1: i^k_{\beta_k}]} W^\square H^\square$, for some $k$, then the latter factorization $G^{[i^k_1: i^k_{\beta_k}]} W^\square H^\square$ give rise to a distinct factorization $S = GW'H'$, where $W'$, $H'$ are the matrices $W$ and $H$ with the appropriate blocks replaced by $W^\square$ and $H^\square$. 

On the other hand, suppose $S = GWH = G\tilde{W}\tilde{H}$. Since $G$ is full-rank, we know that $WH = \tilde{W}\tilde{H}$, in particular, $\tilde{W}\tilde{H}$ is block diagonal with the same block structure as $\tilde{W}\tilde{H}$. Thus, we may write $S = \bigoplus_{k = 1}^K {}_k\tilde{S}$, for 
\[{}_k\tilde{S} = G^{[i^k_1: i^k_{\beta_k}]} \square^{\tilde{W}\tilde{H}}_k = G^{[{i}^k_1: {i}^k_{\alpha_k}]}\tilde{W}^{[{j}^k_1: {j}^k_{{\beta}_k}]}_{[{i}^k_1: {i}^k_{{\alpha}_k}]}\tilde{H}_{[{j}^k_1: {j}^k_{{\beta}_k}]}^{[{l}^k_1: {l}^k_{{\delta}_k}]}. \] 
Since $\tilde{W} \neq W$ and or $\tilde{H} \neq H$, we must have ${}_k{S}\neq {}_k\tilde{S}$ for some $k$. 
\end{proof}

Note that we may have a block-diagonal model $S = GWH$, for $G\neq I_{M\times M}$ where $S$ is not block-diagonalizable by permutation matrices. That is, decomposing $S$ into a direct sum of sub-models is not as trivial as block-diagonalizing $S$. In this sense, the initial choice of $G$ is crucial. 

\begin{ex}
Set the generators of an initial nonnegative simplicial cone to be the column vectors of the following matrix
\begin{align}
G &= \left(
\begin{matrix}
1 & 2 & 3 & 2\\
1 & 3 & 2 & 3\\
2 & 1 & 2 & 2\\
2 & 1 & 2 & 3
\end{matrix}
\right).
\end{align}
Now, choose
\begin{align}
W &=
\left(
\begin{array}{|cc|c@{}c|} 
\cline{1-2}
1 & 2 & & \multicolumn{1}{c}{0} \\
2 & 1 & & \multicolumn{1}{c}{0} \\
1 & \multicolumn{1}{c|}{1} & & \multicolumn{1}{c}{0} \\
\cline{1-2}\cline{4-4}
\multicolumn{1}{c}{0} & \multicolumn{1}{c}{0} & & \multicolumn{1}{|c|}{2} \\
\cline{4-4}
\end{array}
\right),\;\;
H = \left(
\begin{array}{|cc|c@{}c|} 
\cline{1-2}
1 & 2 & & \multicolumn{1}{c}{0} \\
2 & \multicolumn{1}{c|}{1} & & \multicolumn{1}{c}{0} \\
\cline{1-2}\cline{4-4}
\multicolumn{1}{c}{0} & \multicolumn{1}{c}{0} & & \multicolumn{1}{|c|}{1} \\
\cline{4-4}
\end{array}
\right).
\end{align}
So we have that 
\begin{align}
S = GWH =
\left(
\begin{matrix}
    22 &   23   &  4\\
    23 &   25     &6\\
    20 &   19     &4\\
    20  &  19     &6
\end{matrix}
\right),
\end{align}
which decomposes into sub-models according to the following fashion
\begin{align}
S &= {}_1S \oplus {}_2S \\
&= \left(
\begin{matrix}
1 & 2 & 3 \\
1 & 3 & 2 \\
2 & 1 & 2 \\
2 & 1 & 2 
\end{matrix}
\right)
\left(
\begin{matrix}
1 & 2 \\
2 & 1\\
1 & 1
\end{matrix}
\right)
\left(
\begin{matrix}
1 & 2 \\
2 & 1
\end{matrix}
\right)
\oplus \left(
\begin{matrix}
1  \\
1\\
2\\
2
\end{matrix}
\right) (2) (1)
\end{align}
On the other hand, clearly, $S$ is not block-diagonalizable by a permutation matrix. 
\vskip0.01cm
${}$ \hfill $\square$
\end{ex}


\section{Characterizing Type II Non-identifiability}\label{sec:sep}

Given a matrix $S\in \R^{M\times N}$ and $R < M$, if $\rank(S) = \rank_+(S)<R$, say $\rank(S) = R-1$, then $S$ trivially represents a non-identifiable rank-$R$ NMF model of Type II. That is, any rank $R-1$ factorization $S = WH$ can be non-uniquely extended by choosing $\bv \in \R^M$ that is linearly independent from the column vectors of $W$; thus, $S = (W \vert \bv) \left(\begin{matrix}H\\ \zero\end{matrix}\right)$. It follows that nontrivial examples of Type II non-identifiability arises only in cases where the nonnegative rank of $S$ is strictly greater than the rank of $S$. In this section we will study the non-identifiability of a class of examples generated from separable factorial articulation families that satisfies $\rank_+(S) > \rank(S)$. In fact, we give a complete characterize for when for separable factorial articulation families exhibit non-identifiability of Type II.

In the following, let $S = WH$ be a separable factorial articulation family with $P$ parts and $A$ articulations, where $H$ is a binary matrix. We also fix $R = PA$.

\begin{prop}
$R = \rank_+(S) > \rank(S)$.
\end{prop}

\begin{proof}
Note that $\rank_+(S) \leq R$. By Theorem \ref{thm:sep}, we know that the rank-$R$ NMF of $S$ is unique in the subspace spanned by $W$. Suppose that $S$ has a rank-$R'$ NMF, $S = W'H'$, with $R' < R$. Then $\col(W') \subset \col(W)$ and the matrix $W'$ can be non-uniquely extended to give a rank-$R$ NMF, by augmenting it with linearly independent vectors in $\col(W) - \col(W')$. Thus, we must have that $\rank_+(S) \geq R$ and hence $\rank_+(S) = R$. 

Fix orderings of the $P$ parts and $A$ articulations. Assume that the columns in $W$ are lexicographically indexed by the ordered pair of part and articulation, $(p, a)$.  By the Complete Factorial Sampling property (R2 in Theorem \ref{thm:sep}) of $S = WH$, we have that the column vectors of $S$ include
$
\sum_{p} W_{p, a}
$ for each articulation $1 \leq a \leq A$,  also, $W_{1, a+1} + \sum_{p>1} W_{p, a}$ for each $1 \leq a <A$. Finally, by the same token, we have that $W_{1, 1} + \sum_{p>1} W_{p, A}$ appears as a column vector of $S$. But since 
\scriptsize{\begin{align*}
W_{1, 1} + \sum_{p>1} W_{p, A}  = \sum_{a}\sum_{p}W_{p, a} + \sum_{a<A} \left(W_{1, a+1} + \sum_{p>1} W_{p, a}\right)
\end{align*}}
\normalsize it follows that $\rank(S) < R$. 
\end{proof}

We now formalize the intuition behind Example \ref{ex:type2} and describe how separable factorial articulation families can be augmented to provide NMF models that exhibit non-identifiability of Type II.

In the following, we fix orderings of the $P$ parts and $A$ articulations. Assume that the columns in $W$ are lexicographically indexed by part and articulation, $(p, a)$.

\begin{thm}\label{thm:type2}
Given a row vector $\bv \in \R^{R}$ such that $\bv_{\tilde{p}, a}\geq 0$ for a fixed part $\tilde{p}$ and all articulations $a$. Let $W' = \left(\begin{matrix}W\\ \bv \end{matrix}\right)$. Then $S' = W'H$ represents an non-identifiable NMF model of Type II.

In particular, $S' = W'H$ has an infinite number of NMFs.
\end{thm}

\begin{proof}
Let $\epsilon_{\tilde{p}} = \min_{a}\{ \bv_{\tilde{p}, a} \}$ and choose $\tilde{p}' \neq \tilde{p}$. Define $\bv' \in \R^{R}$ by 
\[
\bv'_{p, a} = 
\begin{cases}
\bv_{p,a} - \epsilon_{\tilde{p}}, & p = \tilde{p}\\
\bv_{p,a} + \epsilon_{\tilde{p}}, & p = \tilde{p}'\\
\bv_{p, a}, & \text{otherwise}
\end{cases}
\]
It's straightforward to see that 
$\left(\begin{matrix}
W\\ 
\bv'
\end{matrix}\right)H = 
\left(\begin{matrix}
W\\ 
\bv 
\end{matrix}\right)H$. Given a column vector $\bh$ of $H$, it suffices to check that $\bv'\bh = \bv\bh$. Since $S = WH$ is a complete factorial family, we have that
\[
\bv'\bh = v'_{\tilde{p}, a_{\tilde{p}}} + v'_{\tilde{p}', a_{\tilde{p}'}} + \sum_{p\neq \tilde{p}, \tilde{p}'} v'_{p, a_p}
\]
where $1 \leq a_p, a_{\tilde{p}}, a_{\tilde{p}'}\leq A$. Thus
\[
\bv'\bh = v_{\tilde{p}, a_{\tilde{p}}} - \epsilon_{\tilde{p}} + v_{\tilde{p}', a_{\tilde{p}'}} + \epsilon_{\tilde{p}} + \sum_{p\neq \tilde{p}, \tilde{p}'} v_{p, a_p} = \bv\bh.
\]

Note that since the set of column vectors $W$ is linearly independent, we have that $\col\left(\begin{matrix}
W\\ 
\bv'
\end{matrix}\right)H \neq
\col\left(\begin{matrix}
W\\ 
\bv 
\end{matrix}\right)H$. 

Finally, since the above holds for any choice of $0\leq \epsilon \leq \min_{a}\{ \bv_{\tilde{p}, a} \}$, $S' = W'H$ has an infinite number of NMFs. 
\end{proof}

On the other hand, given a separable factorial articulation family, it is also straightforward to determine if it exhibits non-identifiability of Type II.

\begin{cor}\label{cor:sepfam}
If $W$ contains a row vector $\bv \in \R^{R}$ such that $\bv_{p, a}\geq 0$ for a fixed part $p$ and all articulations $a$. Then the NMF $S = WH$ is unidentifiable of Type II. In particular, $S$ has an infinite number of NMFs.
\end{cor}

\begin{proof}
Let $\tilde{W}$ be the matrix $W$ with the row vector $\bv$ removed. Then $\tilde{W}H$ is a separable factorial articulation family. Thus, by Theorem \ref{thm:type2}, we see that $S = \left(\begin{matrix}\tilde{W}\\\bv \end{matrix} \right)H$ is non-identifiable of Type II and $S$ has an infinite number of NMFs resulting from distributing the values in $v_{p, a}$ to $v_{p', a}$, for $p \neq p'$. 
\end{proof}

The converse of Corollary \ref{cor:sepfam} is also true. Thus, we obtain a complete characterization for when a separable factorial articulation family is non-identifiable of Type II. 

\begin{thm}\label{thm:typeI}
$S = WH$ is non-identifiable of Type II if and only if $W$ contains a row vector $\bv \in \R^{R}$ such that $\bv_{p, a}\geq 0$ for a fixed part $p$ and all articulations $a$.
\end{thm}

\begin{proof}
Let $S = WH = W'H'$. We show that $W' = W$ and $H' = H$, when $W$ does not contain a row containing a part with entirely nonzero values.

Without loss of generality, assume that $W, W'\in \R^{(R+1)\times R}$. Indexing the column of $W$ lexicographically by part and articulation, suppose also that $W$ does not contain a row in which $v_{p, a}$ for a fixed par $p$ and all articulations $a$. Since $W$ is separable, we may assume that $W$ is presented in the form
\begin{align}
W = \left( 
\begin{matrix}
I_{R\times R}\\
W_{\text{lower}}
\end{matrix}
\right)
\end{align}
where $W_{\text{lower}}$ is a row vector. Now, let $S_{\text{upper}}$, $W'_{\text{upper}}$ be sub-matrices of $S$ and $W'$, respectively, each consisting of the first $R$ rows. Then, we have that $S_{\text{upper}} = I_{R\times R} H = W_{\text{upper}}H'$. Since $S_{\text{upper}} = I_{R\times R} H$ is a separable factorial articulation family, we have, from Theorem \ref{thm:sep}, that this factorization is unique in $\col(I_{R\times R}) = \R^R$. Hence, we obtain $W_{\text{upper}} = I_{R\times R}$ and $H = H'$. In other words, $W$ and $W'$ may only differ in their respective last row. 

By assumption, for any part $p$, we must have some articulation $\tilde{a}_p$ such that $(W_{\text{lower}})_{(p, \tilde{a}_p)} = 0$. Without loss of generality, we shall assume that $\tilde{a}_p = 1$. Since $H$ is a complete factorial sampling, the sum $\sum_{p = 1}^P (W_{\text{lower}})_{(p, 1)} = 0$ appears in the last row of $S$. On the other hand, we have that $W_{\text{lower}}H = W'_{\text{lower}}H$ and hence we have that $ (W_{\text{lower}})_{(p, 1)} = 0$ for each part $p$. Thus, we obtain
\begin{align}
(W_{\text{lower}})_{(p', a)} &= (W_{\text{lower}})_{(1, a)} + \sum_{p\neq p'} (W_{\text{lower}})_{(p, 1)} \\
&= (W'_{\text{lower}})_{(1, a)} + \sum_{p\neq p'} (W'_{\text{lower}})_{(p, 1)} \\
&= (W'_{\text{lower}})_{(p', a)}
\end{align}
for each articulation $a$ and each part $p'$. That is, $W = W'$.
\end{proof}

The immediate algorithmic implication of the results in this section is that, for separable factorial articulation families, the entire set of factorizations of $S$ can be obtained from a single factorization, $S = WH$. The following theorem formalizes the intuition in Example (\ref{ex:type2}).

\begin{thm}
Let $\bv \in \R^{R}$ be a row vector of $W$ such that $\bv_{p, a}\geq 0$ for a fixed part $1\leq p\leq P$ and all articulations $1\leq a\leq A$. The points on the standard $p$-simplex uniquely parametrizes a set of nonnegative factorizations of $S$. 
\end{thm}

\begin{proof}
Let $\epsilon_p$ denote $\min_{a} \bv_{p,a}$. The claim follows immediately from the observation that every distribution of $\epsilon_p$ over the $P$ number of parts in $\bv$ gives rise to a alternate factorization of $S$.
\end{proof}

\subsection{Example of Type II Non-identifiability}
The Swimmer data set, first introduced by Lee et al. in a demonstration that NMF is able to produce interpretable decomposition of data into parts \citep{Lee}, is a commonly cited example of a separable factorial articulation family \citep{Donoho}. We describe factorizations performed on this data set, here again, to illustrate the theoretical results in this section. 

The Swimmer dataset contains 256 grey-scale images of a swimmer with all possible combinations of 4 limbs positions \citep{Lee}. Figure \ref{fig:swim}-a is a sample of images from the dataset. The Swimmer dataset is represented as a matrix $M \in \R^{1024\times256}$, each column of which represents an $32 \times 32$ image flattened as a vector. The basis matrix $H\in \R^{1024 \times R}$ learned by an NMF represents a set of images depicting ``parts" that sum-up to each swimmers in the dataset (Figure \ref{fig:swim_nmf}-a).

\begin{figure}[h!]
\begin{center}
\caption{Sample images from variations on the Swimmer dataset}
\label{fig:swim}
\subfigure[With an invariant region (body)]{\includegraphics[width=0.95\columnwidth]{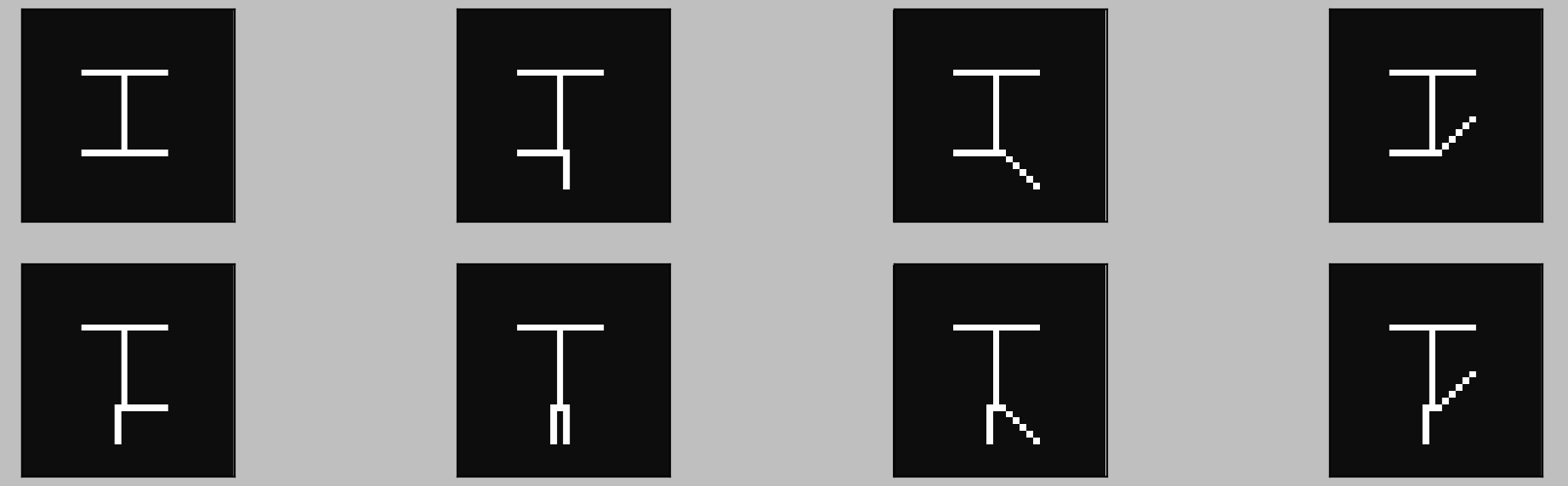}}
\vskip0.1cm
\subfigure[Without the invariant region (body)]{\includegraphics[width=0.95\columnwidth]{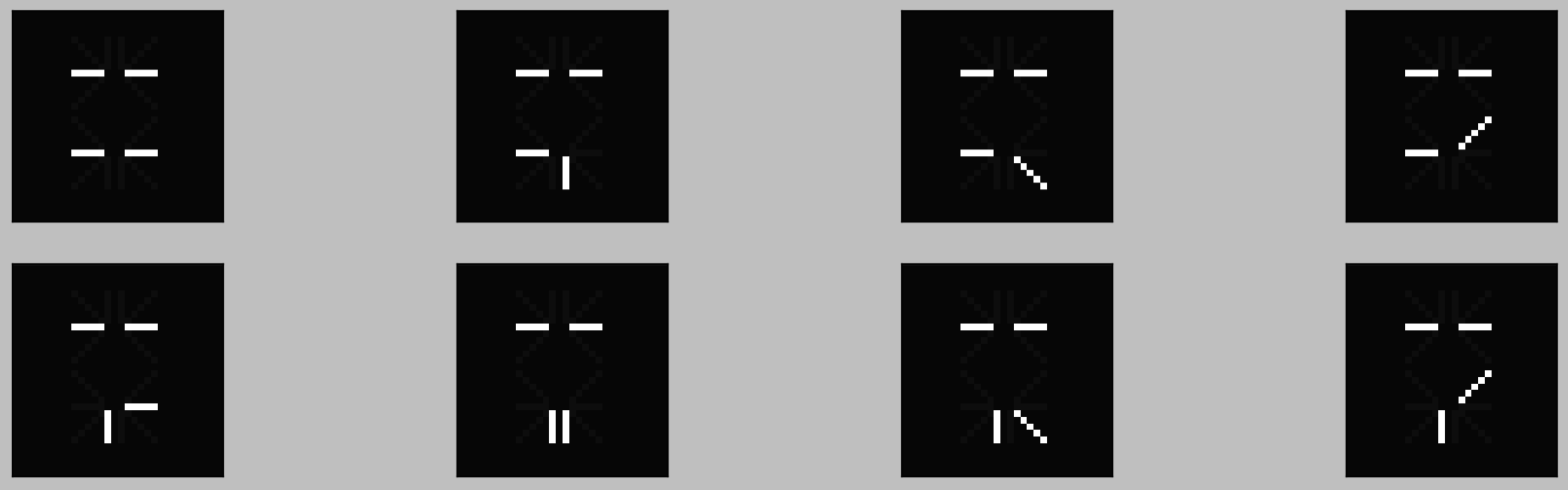}}
\end{center}
\vskip -0.2in
\end{figure} 

Removing the body from each swimmer results in a dataset (Figure \ref{fig:swim}-b) with a unique NMF, the basis of which consists of the four parts (limbs) and four articulations (limb positions). The Swimmer dataset has an infinite number of NMF's, parametrized by distributions of the pixels of the body amongst the four parts (where the pixel intensity is equally distributed amongst the articulations of each part). Note that two different distributions results in bases matrices spanning distinct subspaces. Figure \ref{fig:swim_nmf} shows two bases with different distributions of the body.

\begin{figure}[h!]
\begin{center}
\caption{Basis images for Swimmer provided by NMF}
\label{fig:swim_nmf}
\subfigure[With an invariant region (body)]{\includegraphics[width=0.95\columnwidth]{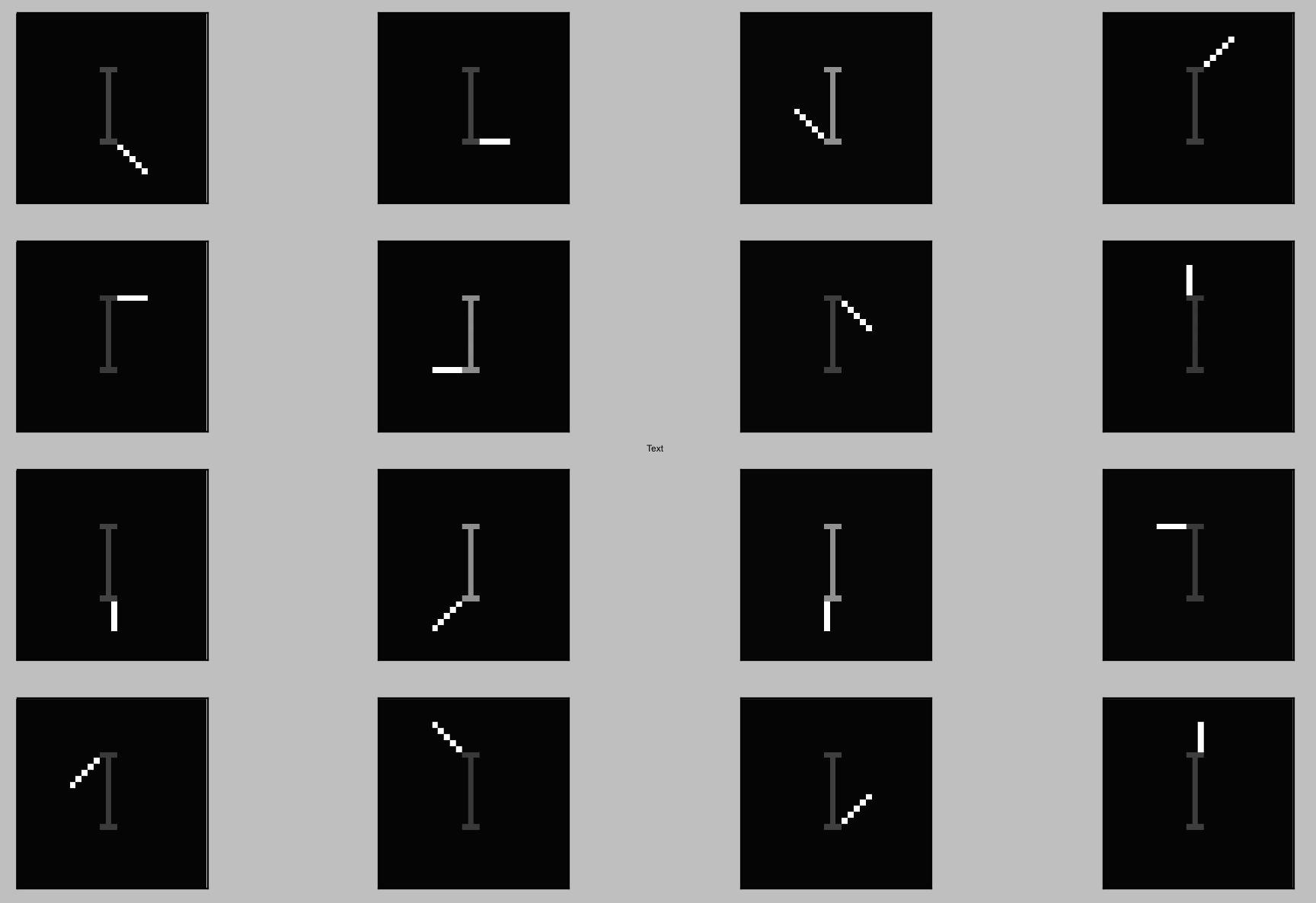}}
\vskip0.1cm
\subfigure[With a different distribution of the invariant region (body)]{\includegraphics[width=0.95\columnwidth]{swimmer_nmf_rand}}
\end{center}
\vskip -0.2in
\end{figure}


\section{Discussion}
Whereas existing work in literature have provided necessary and/or sufficient conditions for the uniqueness of specific classes of models \citep{Donoho, Laurberg, Huang}, there has not been a context in which these models can be related in a unified fashion. In particular, in literature on the uniqueness of NMFs, it is always assumed that $\rank_+(S) = \rank(S)$ and therefore two distinct factorizations of $S$ differ by an $R\times R$ change of bases matrix $Q$ (Type I). However, this assumption excludes many models generated from separable factorial articulation families (the Swimmer example of \citep{Donoho}, for example), whose distinct factorizations differ by an $M\times M$ change of bases matrix (Type II). In this paper, we have introduced a complete framework for characterizing the non-identifiability of NMF models. Such a framework makes possible and provides the essential foundation for principled explorations of the non-identifiability of NMFs. In this section, we describe a few directions of these explorations.  

\subsection{Characterization of Unidentifiable Indecomposable Models}
The results of Section \ref{sec:decomp} imply that the non-identifiability of block-diagonal models can be characterized by the non-identifiability of indecomposable models. An interesting future direction would be to combinatorially classify indecomposable models and to fully characterize the non-identifiability of classes of simple models.   Then one could demonstrate that block-diagonal models can generally be reasonably approximated by a direct sum of a subset of aforementioned simple models. One of the goals for this characterization is the systematic generation of alternate solutions in NMF applications. 

Toward this end, in this work, we show that the block-diagonal decomposition of an NMF model is obtained by finding an initial factorization $S = GW$, where $W$ is a block-diagonal matrix (with multiple blocks).  This naturally suggests that we could efficiently find the block structure in Type I unidentifiable models if we could efficiently discover the super-cone $G$.  We hypothesize that the geometry of the columns of $G$ may provide insight into efficient approximations for $G$.  

\subsection{Model Checking for Type II Non-identifiability}
In Section \ref{sec:sep}, we provided a complete characterization of non-identifiability for a class of separable factorial articulation families. For these results to be easily applicable, one would ideally have a method of determining whether a given matrix $S$ is generated by a separable factorial articulation family that is robust under the presence of noise.  Exploiting the geometry of the column vectors of $S$ to verify whether it comes from a factorial articulation family is the subject of current work. 

\subsection{Approximate NMFs}
Finally, this work has only addressed exact factorizations of matrices. However, in the presence of noise, matrix factorization is performed with some tolerance for error. That is, we usually seek factorizations such that $WH$ such that $\|S - WH\|_F \leq \epsilon$, for $\epsilon>0$. Here, a rigorous understanding of the exact solution space of the underlying NMF model provides insight about the geometry and topology of the set of approximate solutions. For example, the number and forms of exact solutions correspond to the number and location of the modes in the approximate solution space. Furthermore, the approximate solution space, say for the matrix $W$, can be realized the union of closed $\epsilon$-ball centered at each $W$ that is a factor for an exact solution. Thus, we expect the geometry and topology of the exact solution space to dictate that of the approximate solution space.  We expect that our characterization of the exact case will accelerate future work in characterizing non-identifiability in, as well as traversing the solution space of, approximate NMF problems.

\section*{Acknowledgments}   
We thank Arjumand Masood for many helpful discussions and insights.


\bibliographystyle{icml2015}
\bibliography{arxiv}

\end{document}